%% file: main.tex
\documentclass[10pt]{article}

\usepackage[preprint]{tmlr}

\input{math_commands.tex}

\input{mathdef}

\usepackage{hyperref}
\usepackage{url}
\usepackage{graphicx, color}

\usepackage[noorphans]{quoting}
\usepackage[square,sort,comma,numbers]{natbib}

\usepackage{amsmath}
\usepackage{amssymb}
\usepackage{mathtools}
\usepackage{amsthm}
\usepackage{bbm}
\usepackage{booktabs}
\usepackage{tabularx}
\usepackage{subcaption}
\usepackage{wrapfig}

\usepackage{caption}

\usepackage[most]{tcolorbox}

\usepackage[ruled, algo2e, vlined]{algorithm2e}
\SetKwInput{KwInput}{Input}
\SetKwInput{KwOutput}{Output}

\theoremstyle{plain}
\newtheorem{theorem}{Theorem}[section]

\theoremstyle{definition}

\theoremstyle{remark}

\allowdisplaybreaks

\newcommand{\visiblespacechar}{\texttt{\textvisiblespace}}
\newcommand{\visibletabchar}{\ensuremath{\rightarrow\mkern-7mu\vert}}

\usepackage{listings}
\usepackage{xcolor}
\definecolor{codegreen}{rgb}{0,0.6,0}
\definecolor{codegray}{rgb}{0.5,0.5,0.5}
\definecolor{codepurple}{rgb}{0.58,0,0.82}
\definecolor{backcolour}{rgb}{1,1,1}

\lstdefinestyle{mystyle}{
    backgroundcolor=\color{backcolour},   
    commentstyle=\color{codegreen},
    keywordstyle=\color{magenta},
    numberstyle=\tiny\color{codegray},
    stringstyle=\color{codepurple},
    basicstyle=\ttfamily\footnotesize, % font size
    breakatwhitespace=false,         
    breaklines=true,                 
    captionpos=b,                    
    keepspaces=true,                 
    numbers=left,                    
    numbersep=5pt,                  
    showspaces=false,                
    showstringspaces=false,
    showtabs=false,                  
    tabsize=4,
    frame=lines      % lines at top and bottom
}

\title{Improving the Diversity of LLM Outputs without a Trade-off}

\author{\name Ryoma Sato \email rsato@nii.ac.jp \\
  \addr National Institute of Informatics
}

\begin{document}

\maketitle

\begin{abstract}
  We propose DAST (Diversifying Arithmetic Sampling with TokenTour), a method that increases the diversity of LLM outputs without any change to the marginal distribution and with negligible generation-time overhead (a few microseconds). We observe that token IDs are often arranged in a meaningless order and reassign them so that tokens with similar meanings appear consecutively. This can be done in advance in a few hundred seconds per model, and the resulting order can be reused for all subsequent generations. By combining this order with arithmetic sampling (or quasi-Monte Carlo methods), we make similar tokens less likely to be generated across runs while preserving the distribution. Our method not only produces qualitatively good ideas but also significantly improves performance on the downstream task of ProtoQA.
\end{abstract}

\section{Introduction}

Real-world tasks do not always have a unique answer. Many LLM benchmarks evaluate accuracy by having the model produce a single answer, but in practice, it is desirable for an LLM to generate multiple ideas for the user to consider. Simply changing the random seed and generating again, however, may produce similar ideas.

Another setting in which diversity matters is solving problems with verifiable solutions~\cite{brown2024large,wang2026effectsampling,lamont20253d}. When it is possible to verify whether a problem has been solved, an effective strategy is to have an LLM generate multiple candidate solutions and adopt one that solves the problem, if any. Generating candidates that are as diverse as possible increases the likelihood that at least one is correct. Even in verifiable settings, as discussed by \citet{lee2025how}, obtaining a single output that passes the tests is not always sufficient: diversity among the passing outputs may itself be important.
Diverse outputs are also useful for reinforcement learning with verifiable rewards~\cite{yao2025diversity,li2026quasimotto,li2026beyond,li2026setpo,tuyls2026representation}. In reinforcement learning with verifiable rewards, training can stall when the outputs are either all correct or all incorrect~\cite{yu2025dapo,le2026no,xu2026prune,jiang2026tapo}, and diversifying LLM outputs alleviates this problem. Even when solutions are not necessarily verifiable, selecting from diverse outputs using an automatic metric can improve performance~\cite{freitag2022high}.

Many methods have been proposed for generating diverse ideas. Approaches include changing the sampling temperature, perturbing the prompt~\cite{wang2025diversified}, sampling without replacement~\cite{kool2019stochastic,shi2020incremental}, and modifying the beam search criterion~\cite{vijayakumar2018diverse,meister2021determinantal}. These approaches, however, generally have drawbacks such as changing the original distribution or increasing generation time.

Arithmetic Sampling~\cite{vilnis2023arithmetic} is one of the rare methods that can increase diversity without changing the marginal distribution or increasing generation time, while allowing parallelization just as in i.i.d. generation. The method constructs an arithmetic code from the LLM's generation distribution. To generate $K$ samples, it first samples a reference position $p \sim \text{Unif}([0, 1])$ and then sets the position of sample $i = 1, 2, \ldots, K$ to $p_i = (p + \frac{i - 1}{K}) \mod 1$. This guarantees that the $K$ samples are evenly spaced over $[0, 1]$. Figure \ref{fig: arithmetic} illustrates next-token generation with $K = 3$. When the next-token distribution is represented on a circle, the $K$ random numbers are placed exactly 120 degrees ($\frac{360}{K}$) apart. With i.i.d. sampling, all random numbers may fall in the same region, whereas arithmetic sampling always spaces them evenly, making repeated generation of the same token less likely. Even with this arrangement, each individual arm is sampled uniformly over 360 degrees, so the marginal distribution is unaffected. The generation overhead is also negligible.

\begin{figure}[t]
  \centering
  \includegraphics[width=0.6\textwidth]{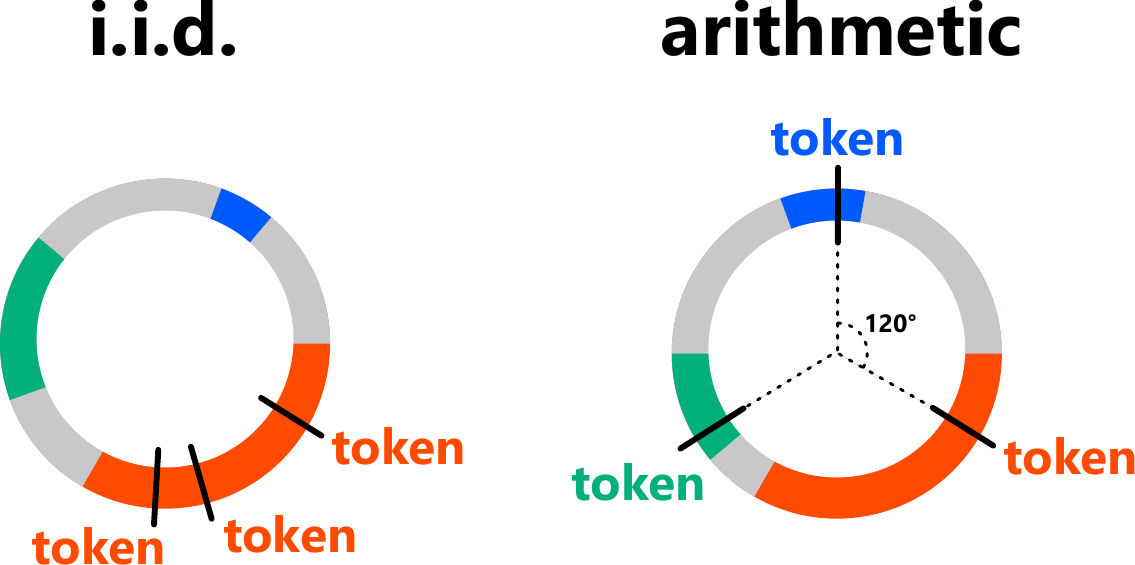}%
  \caption{Comparison of i.i.d. and arithmetic sampling when generating $K = 3$ samples. With i.i.d. sampling, all three may fall on the same token. Arithmetic sampling places the three random numbers exactly 120 degrees ($\frac{1}{3}$ of a turn) apart, making them less likely to fall on the same token. Note that even with arithmetic sampling, each arm is sampled uniformly over 360 degrees.} \label{fig: arithmetic}
\end{figure}

Arithmetic sampling, however, does not account for the ordering of the vocabulary around the circle. An LLM's vocabulary can contain tokens with the same meaning but different token IDs. This can arise for various reasons, including singular and plural forms such as gift and gifts, differences in capitalization such as tour and Tour, and synonyms such as tale and story. Thus, even when arms are 120 degrees apart and fall on different token IDs, those IDs may correspond to tokens with the same meaning (Figure \ref{fig: dast}, left). In such cases, arithmetic sampling can still generate tokens with the same meaning. In particular, in the next-token distribution, implausible tokens have probabilities close to $0$ and occupy arcs too short to see, while a set of similar, plausible tokens occupies a large area. Empirically, several kinds of similar tokens are often scattered around the circle, making this outcome more likely.

We observe that vocabulary ordering offers an opportunity for optimization ``for free.'' The order of vocabulary IDs has no meaning to an LLM and has therefore been overlooked. It is only when arithmetic sampling is used that the order becomes meaningful because of the relationships between the arms. Prior work recognized this opportunity but did not act on it. Indeed, \citet{parashar2025quasi} state that ``we limited ourselves to randomizing the vocabulary ordering and then performing arithmetic sampling,'' but did not optimize the vocabulary order. We are the first to consider how to assign vocabulary IDs for arithmetic sampling.

Finding an optimal vocabulary order is not straightforward. There are $n!$ permutations, and optimizing a permutation is often nontrivial, as in the traveling salesman problem. We consider obtaining an ordering of an LLM's vocabulary using WordTour \cite{sato2022word}, which optimizes the order of word embeddings to arrange them in one dimension. This approach has the advantages of requiring no calibration data and allowing the order to be reused once constructed.

We show that even at the scale of an LLM's vocabulary, a well-organized ordering can be obtained in a practical amount of time. We also show that using this vocabulary order increases the diversity of LLM outputs ``for free'' compared with the standard vocabulary order. In particular, our method significantly improves performance on the downstream task of ProtoQA, which measures the ability to generate diverse and valid outputs.

\begin{figure}[t]
  \centering
  \includegraphics[width=0.6\textwidth]{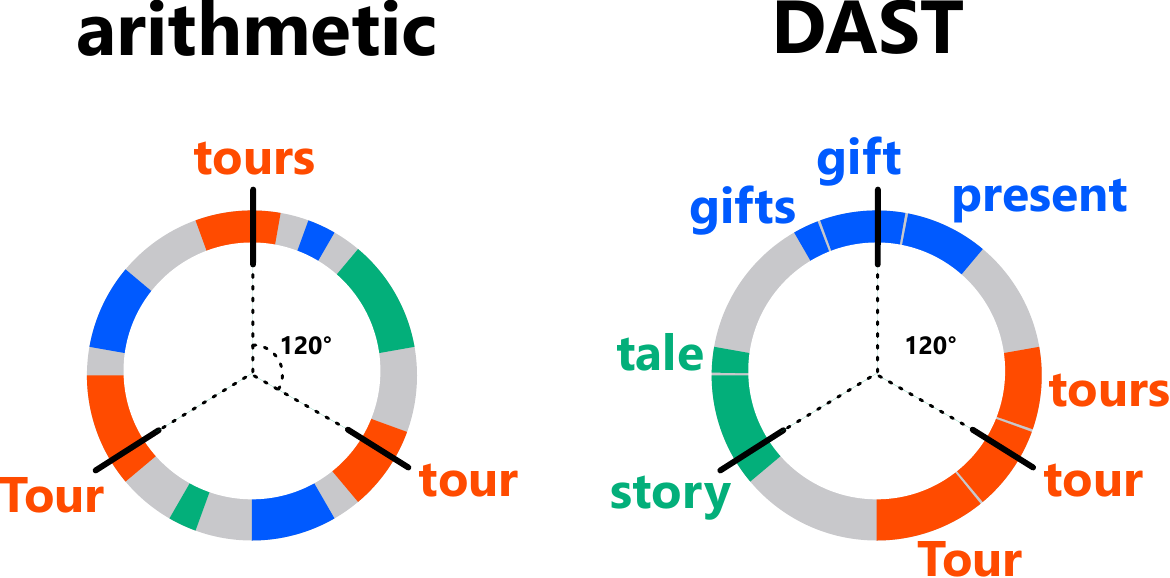}%
  \caption{Comparison of arithmetic sampling and our proposed method, DAST. LLM vocabularies generally have a disordered arrangement. Arms separated by 120 degrees are less likely to fall on exactly the same token ID, but may fall on similar tokens with different IDs. This is particularly likely in the next-token probability distribution, where irrelevant tokens have probabilities close to 0, while similar tokens that are plausible as the next token have large probabilities (areas). DAST uses WordTour to place tokens with similar meanings next to one another, making arms separated by 120 degrees more likely to fall on tokens with different meanings. Since this simply reassigns vocabulary IDs in advance, exactly the same inference program can produce exactly the same marginal distribution. When $K > 1$ samples are generated simultaneously, the only difference is that the samples are more diverse.} \label{fig: dast}
\end{figure}

\section{Proposed Method (DAST)}

Simply put, our proposed method, DAST, reorders the vocabulary using WordTour and applies arithmetic sampling.

Specifically, let $\mathcal{V} = \{1, 2, \ldots, n\}$ be the vocabulary of an LLM, and let $\boldv_i$ be the (input) embedding of token $i$. We reorder the vocabulary so that similar tokens are adjacent by solving the following optimization problem:\begin{align}
  \min_{\pi\colon \mathcal{V} \to \mathcal{V}} \sum_{i=1}^n \|\boldv_{\pi(i)} - \boldv_{\pi(i+1)}\| \label{eq: wordtour}
\end{align}
Here, $\pi(n + 1) = \pi(1)$. That is, we arrange the vocabulary around a circle, corresponding to the placement of samples using $\text{mod} 1$ in arithmetic sampling. The original WordTour solves optimization problem (\ref{eq: wordtour}) using the LKH solver~\cite{lkh_webpage}, whereas DAST constructs an initial tour with the nearest-neighbor heuristic and refines it with 2-OPT. We found that 2-OPT is sufficient for arithmetic sampling. Although the LKH solver can slightly reduce the objective value of problem (\ref{eq: wordtour}), the difference is small and does not translate into improvements in the diversity of arithmetic sampling or downstream task performance. Since 2-OPT closes almost the entire gap between the LLM's standard vocabulary order and the optimal order, is simple, and runs quickly, it is well suited to DAST's purpose.

This reordering requires only the model parameters. The absence of a need for calibration data is an advantage. For example, one could also consider reordering tokens based on the model's activation patterns, but this would require calibration data and would be more expensive. The input embeddings used by DAST come with the model ``for free,'' making reordering inexpensive.

Note that this reordering must be performed separately for each model because models may have different vocabularies. Once it has been performed for a model, however, the order can be reused for all of that model's outputs. For typical vocabularies of tens of thousands of tokens, reordering takes a few minutes. We also release the vocabulary orders used in our experiments, so users of these models do not need to perform the reordering themselves.

Once the reordered vocabulary is available, DAST performs arithmetic sampling as usual. It is extremely simple, but simplicity is a virtue. The reordered vocabulary can be plugged in without even changing the generation program or incurring generation overhead, yet DAST does increase the diversity of LLM outputs.

DAST inherits the desirable properties of arithmetic sampling. In particular, it guarantees that the desired result will be obtained if the number of samples $K$ is sufficiently large. Specifically, suppose that a desirable output $y$ exists for a prompt $x$. Assume that the user does not know $y$ in advance but can recognize it as desirable when the LLM presents $y$. Such settings range from casual situations, such as brainstorming gifts for a friend, to rigorous ones, such as RLVR. In these settings, the LLM should generate outputs that are as diverse as possible so that at least one matches $y$. The following theorem holds.

\begin{theorem}
  Suppose that all token probabilities of the LLM are positive, as is the case for typical LLMs with softmax outputs. For any token sequence $x$ and any token sequence $y$, there exists a number of samples $K$ such that $y$ is guaranteed to be obtained. That is,\begin{align}
  \Pr_{y_1, \ldots, y_K \sim \text{DAST}(x)}\left[\exists i \in [K], y_i = y\right] = 1
\end{align} holds. In contrast, no such $K$ exists for i.i.d. sampling.
\end{theorem}

\begin{proof}
  Let $K = \lceil \frac{1}{\text{P}[y]} \rceil$. In the arithmetic code, $y$ is assigned a contiguous interval of length $\text{P}[y]$. Since the distance between the arms in arithmetic sampling is $\frac{1}{K} \leq \text{P}[y]$, at least one arm falls in the interval for $y$, regardless of the initial position of the arms. Thus, there exists an $i \in [K]$ such that $y_i = y$. In contrast, with i.i.d. sampling, $\Pr[y_i = y] = \text{P}[y]$ and $\Pr[\forall i \in [K], y_i \neq y] = (1 - \text{P}[y])^K > 0$, so for any $K$, there is always a possibility that $y$ is not generated.
\end{proof}

\section{Experiments}

We demonstrate DAST's speed, its qualitative effectiveness, and its quantitative effectiveness on ProtoQA.

\subsection{Experimental Setup}

\begin{table}[t]
\centering
\caption{Vocabulary sizes of the models.}
\label{tab:vocab_size}
\small
\begin{tabular*}{\columnwidth}{@{\extracolsep{\fill}}lrrrr@{}}
\toprule
Model
& SmolLM2-1.7B
& Qwen2.5-1.5B
& Qwen2.5-3B
& Qwen2.5-7B \\
\midrule
$|V|$
& 49,152
& 151,936
& 151,936
& 152,064 \\
\bottomrule
\end{tabular*}
\end{table}

We conduct experiments using SmolLM2-1.7B-Instruct~\cite{allal2025smollm2}, Qwen2.5-1.5B-Instruct, Qwen2.5-3B-Instruct, and Qwen2.5-7B-Instruct~\cite{qwen2024qwen2}; in the tables and figures we omit the \texttt{-Instruct} suffix for brevity. These models span two model families and a range of sizes, providing a diverse testbed. Their vocabulary sizes are listed in Table \ref{tab:vocab_size}. The vocabularies are relatively large, making scalability important. We compare i.i.d. sampling, arithmetic sampling with the standard vocabulary order, and DAST. All of these preserve the marginal distribution; they are among the rare approaches that allow diversity to be increased without side effects, in a sense.

ProtoQA~\cite{boratko2020protoqa} is a dataset of commonsense questions without a unique answer. For example, 100 people are asked a question such as ``Name something that people usually do before they leave for work?'' Of these, 43 answer Shower, 30 answer Breakfast, and 7 answer Lock door. The dataset measures the ability to generate diverse ideas by evaluating answer coverage without counting repeated answers more than once. For example, if an LLM is asked to generate $K = 3$ answers and responds with Shower, Shower, Shower, it receives 43 points, the score for the unique answer. If it instead responds with Shower, Breakfast, Lock door, it receives 43 + 30 + 7 = 80 points. Thus, generating more diverse answers can yield a higher score. We use this dataset to evaluate the diversity of DAST.

\subsection{Speed Comparison}

\begin{table}[t]
    \centering
    \caption{The two types of overhead in DAST. ``One-time'' is the time required to construct the vocabulary order (nearest-neighbor search, tour construction, and 2-OPT combined), which can then be reused for all subsequent generations. The remaining columns show the time required to generate one token. ``Common'' denotes operations independent of the sampling method, such as the forward pass. DAST's per-generation overhead relative to i.i.d. sampling is about 2 microseconds, less than 0.4\% of the forward-pass time. Measurements were taken on an NVIDIA RTX PRO 6000 Blackwell GPU and an AMD EPYC 9534 CPU.}
    \label{tab:speed}
    \small
    \begin{tabular*}{\columnwidth}{@{\extracolsep{\fill}}lrrrrrr@{}}
    \toprule
    & One-time & Common
    & \multicolumn{4}{c}{Sampling overhead} \\
    \cmidrule(l){4-7}
    Model & (Order construction) & (Forward pass, etc.) & i.i.d. & token ID & DAST & DAST $-$ i.i.d. \\
    \midrule
    SmolLM2-1.7B  &  21.3 s &   503 $\mu$s & 8.6 $\mu$s & 10.1 $\mu$s & 10.5 $\mu$s & $+$1.9 $\mu$s \\
    Qwen2.5-1.5B  &  83.5 s &   530 $\mu$s & 8.5 $\mu$s &  9.9 $\mu$s & 10.3 $\mu$s & $+$1.8 $\mu$s \\
    Qwen2.5-3B    & 105.6 s &   764 $\mu$s & 8.6 $\mu$s & 10.0 $\mu$s & 10.6 $\mu$s & $+$2.0 $\mu$s \\
    Qwen2.5-7B    & 228.0 s & 1,215 $\mu$s & 8.9 $\mu$s & 10.4 $\mu$s & 10.9 $\mu$s & $+$2.0 $\mu$s \\
    \bottomrule
    \end{tabular*}
\end{table}

We measure the speed of DAST. DAST has two types of overhead: the one-time construction of the order and arithmetic sampling at each generation. Table \ref{tab:speed} reports the corresponding runtimes. Even the one-time computation can be completed in a few minutes, and the generation overhead is only a few microseconds, making it negligible.

\subsection{Vocabulary Ordering}

\begin{table}[t]
    \centering
    \caption{Objective values (\ref{eq: wordtour}) of the vocabulary orders (lower is better). The ``Held--Karp lower bound'' is a rigorous lower bound obtained from the dual problem; the true optimum is at least this value. Thus, the gap between DAST and the true optimum is bounded above by the ``DAST / lower bound'' column.}
    \label{tab:tour_objective}
    \small
    \begin{tabular*}{\columnwidth}{@{\extracolsep{\fill}}lrrrrr@{}}
    \toprule
    Model & Standard vocabulary IDs & DAST & Held--Karp lower bound & DAST / lower bound \\
    \midrule
    SmolLM2-1.7B & 212,363 & 151,905 & 148,736 & $+$2.13\% \\
    Qwen2.5-1.5B & 203,768 & 148,881 & 145,729 & $+$2.16\% \\
    Qwen2.5-3B   & 223,011 & 163,290 & 159,730 & $+$2.23\% \\
    Qwen2.5-7B   & 169,399 & 146,950 & 145,173 & $+$1.22\% \\
    \bottomrule
    \end{tabular*}
\end{table}

We analyze the vocabulary orders obtained by optimization. Table \ref{tab:tour_objective} reports the objective values for the models' standard vocabulary orders and the orders obtained by DAST. A lower bound for optimization problem (\ref{eq: wordtour}) can also be obtained from the dual problem. We report this value as a reference for assessing optimality. The models' standard vocabulary orders leave considerable room for optimization, and DAST successfully brings the objective values close to the lower bound.

\begin{table}[t]
    \centering
    \caption{The models' standard token ID orders and the token orders produced by DAST. Only selected segments are shown. The standard orders are completely disordered, whereas WordTour places similar tokens in contiguous segments.}
    \label{tab:wordtour_tokens}

    \small
    \setlength{\tabcolsep}{5pt}
    \renewcommand{\arraystretch}{1.08}

    \begin{tabularx}{\linewidth}{
        >{\raggedright\arraybackslash}p{0.18\linewidth}
        >{\raggedright\arraybackslash}X
        >{\raggedright\arraybackslash}X
    }
        \toprule
        \textbf{Model}
        & \textbf{Standard ID}
        & \textbf{DAST} \\
        \midrule

        \textbf{SmolLM2-1.7B}
        &
        \ttfamily
        \visiblespacechar Rim \newline
        \visiblespacechar equates \newline
        \visiblespacechar streaks \newline
        \visiblespacechar pharmacists \newline
        \textquotedbl... \newline
        Luck \newline
        dialog \newline
        jas \newline
        \visiblespacechar REG \newline
        \visiblespacechar Ngu \newline
        \visiblespacechar mixer \newline
        \visiblespacechar Jesuits
        &
        \ttfamily
        \visiblespacechar exceedingly \newline
        \visiblespacechar extremely \newline
        \visiblespacechar incredibly \newline
        \visiblespacechar exceptionally \newline
        \visiblespacechar extraordinarily \newline
        \visiblespacechar extraordinary \newline
        \visiblespacechar remarkable \newline
        \visiblespacechar incredible \newline
        \visiblespacechar amazing \newline
        \visiblespacechar astounding \newline
        \visiblespacechar astonishing \newline
        \visiblespacechar startling
        \\

        \midrule

        \textbf{Qwen2.5-1.5B}
        &
        \ttfamily
        .DB \newline
        \visiblespacechar popularity \newline
        \visiblespacechar gew \newline
        \visiblespacechar impr \newline
        setValue \newline
        FLAG \newline
        \visibletabchar max \newline
        \visiblespacechar bake \newline
        wy \newline
        \visiblespacechar Economic \newline
        \visiblespacechar encontr \newline
        \visiblespacechar fname
        &
        \ttfamily
        \visiblespacechar huge \newline
        \visiblespacechar enormous \newline
        \visiblespacechar immense \newline
        \visiblespacechar tremendous \newline
        \visiblespacechar tremendously \newline
        \visiblespacechar immensely \newline
        \visiblespacechar enormously \newline
        \visiblespacechar hugely \newline
        \visiblespacechar greatly \newline
        \visiblespacechar vastly \newline
        \visiblespacechar radically \newline
        \visiblespacechar drastically
        \\

        \bottomrule
    \end{tabularx}
\end{table}

Table \ref{tab:wordtour_tokens} presents representative segments selected from the orders for SmolLM2-1.7B and Qwen2.5-1.5B. While the models' standard orders are disordered, DAST places tokens with the same or similar meanings in contiguous segments. The complete orders produced by DAST can be viewed at the following links:

\begin{description}
  \item[SmolLM2-1.7B] \url{https://raw.githubusercontent.com/joisino/dast/refs/heads/main/HuggingFaceTB_SmolLM2-1.7B-Instruct.txt}
  \item[Qwen2.5-1.5B] \url{https://raw.githubusercontent.com/joisino/dast/refs/heads/main/Qwen_Qwen2.5-1.5B-Instruct.txt}
  \item[Qwen2.5-3B] \url{https://raw.githubusercontent.com/joisino/dast/refs/heads/main/Qwen_Qwen2.5-3B-Instruct.txt}
  \item[Qwen2.5-7B] \url{https://raw.githubusercontent.com/joisino/dast/refs/heads/main/Qwen_Qwen2.5-7B-Instruct.txt}
\end{description}

For all models, DAST successfully places tokens with the same meaning close to one another. These lists save users the effort of ordering the vocabulary themselves.

\subsection{ProtoQA}

\begin{figure}[t]
  \centering
  \includegraphics[width=\textwidth]{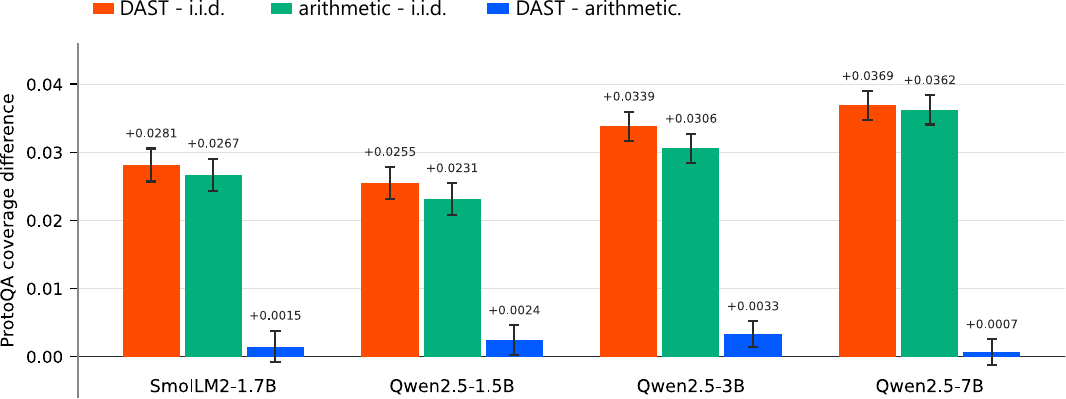}%
  \caption{Quantitative results on ProtoQA. A positive orange bar on the left indicates that DAST outperforms i.i.d. sampling. A positive green bar in the middle indicates that arithmetic sampling outperforms i.i.d. sampling. An orange bar on the left that is higher than the green bar in the middle, as well as a positive blue bar on the right, indicates that DAST outperforms arithmetic sampling. Error bars show 95 percent confidence intervals.} \label{fig: protoqa}
\end{figure}

Figure \ref{fig: protoqa} shows the experimental results on ProtoQA. We sample $3$ answers from the LLM using each method and evaluate their coverage of human answers. Following ProtoQA's Max Answers@3, the score is normalized so that it is 1 when the three largest answer clusters are covered and 0 when none are covered. To examine the differences between methods in detail, the figure shows score differences and their confidence intervals. DAST and arithmetic sampling achieve approximately 2 to 4 percentage points greater coverage than i.i.d. sampling, and DAST further improves coverage over arithmetic sampling by up to 0.33 percentage points. For Qwen2.5-1.5B and Qwen2.5-3B, the differences exceed the uncertainty indicated by the confidence intervals. This is a substantial gain given that DAST can be introduced with almost no overhead.

\subsection{Qualitative Example}

\begin{figure}[t]
  \centering
  \includegraphics[width=0.6\textwidth]{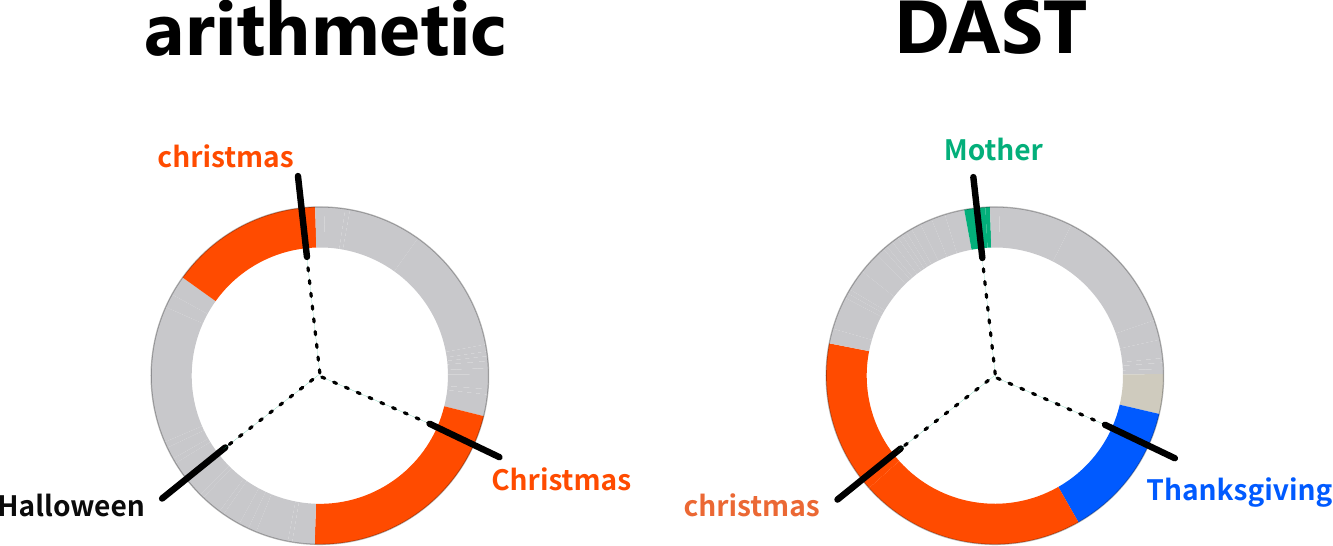}%
  \caption{Example answers from Qwen2.5-0.5B-Instruct, a smaller model used here for illustration, to the ProtoQA question ``name a day of the year when most people call home.'' The figure uses the actual token orders and arc lengths proportional to the output probabilities. In the standard token order, christmas and Christmas are far apart, and two arms of arithmetic sampling hit these two tokens, making the first tokens of the answers Christmas, christmas, Halloween. In DAST, the Christmas and christmas tokens are close together and are therefore not selected simultaneously; the first tokens of the outputs are christmas, Thanksgiving, Mother.} \label{fig: illustrative}
\end{figure}

Figure \ref{fig: illustrative} presents an end-to-end qualitative example using the answers of Qwen2.5-0.5B-Instruct, a smaller model used here for illustration, to the ProtoQA question ``name a day of the year when most people call home.'' In the standard token order, christmas and Christmas are far apart, and two arms of arithmetic sampling hit these two tokens, making the first tokens of the answers Christmas, christmas, Halloween. In DAST, the Christmas and christmas tokens are adjacent and are therefore not selected simultaneously; the first tokens of the outputs are christmas, Thanksgiving, Mother. In this way, DAST is less likely to select distinct tokens with the same meaning, allowing it to produce more diverse outputs.

\section{Related Work}

This section discusses the diversity of LLM outputs. The simplest way to increase diversity is to raise the sampling temperature~\cite{ippolito2019comparison,caccia2020language}. Although this increases diversity, it generally reduces quality~\cite{caccia2020language}. This contrasts with our focus on improving diversity without any loss of quality. This trade-off is common to most diversification methods. \citet{wang2025diversified} showed that, when obtaining multiple answers to the same question, perturbing the prompt produces more diverse outputs than generating repeatedly from exactly the same prompt. This method also exhibits a trade-off between diversity and quality depending on the strength of the perturbation. In a similar direction, \citet{wang2025multilingual} improve diversity by constructing prompts in multiple languages and cultural contexts. Methods that intervene in representations, such as G2~\cite{ruan2025g2}, STARS~\cite{zhu2026exploring}, and SemDiD~\cite{shi2025semantic}, exhibit a similar trade-off. A related but distinct line of work is sampling without replacement from language models~\cite{kool2019stochastic,shi2020incremental}. \citet{kool2019stochastic} proposed an efficient method for sampling without replacement from sequence models using the Gumbel-Top-k trick and lazy evaluation of Gumbel samples on the sequence tree. \citet{shi2020incremental} proposed a more efficient method for sampling without replacement using a trie. Although sampling without replacement may slightly improve diversity over sampling with replacement, it does not directly prevent duplicates among semantically equivalent sequences. Another popular heuristic is to modify the beam search criterion~\cite{vijayakumar2018diverse,meister2021determinantal}. \citet{vijayakumar2018diverse} encourage diversity by adding a scoring term that favors candidates less similar to previously selected samples. Determinantal Beam Search \cite{meister2021determinantal} uses ideas from determinantal point processes to select beam search candidates that have high likelihoods while being dissimilar to one another. Although these methods empirically increase diversity, they generally provide no guarantee of diversity in the final outputs and change the original distribution. Simpler heuristics include prompting an LLM to produce different answers~\cite{troshin2025asking}, but these also generally offer no guarantee of diversity in the final outputs and change the original distribution. They also make generation difficult to parallelize. \citet{wong2025simplestrat} improve diversity through step-by-step stratified sampling, but this can also change the distribution and generally increases processing time. Other studies seek to improve diversity through post-training, such as \citet{chung2025modifying} and \citet{li2025jointly}, but these require a certain amount of computation and, naturally, change the distribution. The same holds for methods such as BACo~\cite{wang2026optimizing}, which use more diverse base models as appropriate. A parallel line of work studies diversification in diffusion language models~\cite{wu2026time,lamont2026free}. Note that our paper assumes autoregressive LLMs and thus differs in scope from these studies.

\section{Conclusion}

We proposed DAST, a method that can increase the diversity of LLM outputs without changing the marginal distribution or increasing generation time, while allowing parallelization. DAST addresses the overlooked role of LLM vocabulary ordering in the literature and improves diversity by explicitly constructing an arrangement. DAST improved the arrangement and increased diversity both qualitatively and quantitatively.

\bibliography{main}
\bibliographystyle{abbrvnat}

\appendix

\end{document}

%% file: math_commands.tex
\usepackage{amsmath,amsfonts,bm}

\def\eqref#1{equation~\ref{#1}}
\def\1{\bm{1}}

\DeclareMathAlphabet{\mathsfit}{\encodingdefault}{\sfdefault}{m}{sl}
\SetMathAlphabet{\mathsfit}{bold}{\encodingdefault}{\sfdefault}{bx}{n}

%% file: mathdef.tex
\newcommand{\boldv}{{\boldsymbol{v}}}

%% file: main.bbl
\begin{thebibliography}{39}
\providecommand{\natexlab}[1]{#1}
\providecommand{\url}[1]{\texttt{#1}}
\expandafter\ifx\csname urlstyle\endcsname\relax
  \providecommand{\doi}[1]{doi: #1}\else
  \providecommand{\doi}{doi: \begingroup \urlstyle{rm}\Url}\fi

\bibitem[Allal et~al.(2025)Allal, Lozhkov, Bakouch, Blázquez, Penedo, Tunstall, Marafioti, Kydlícek, Lajarín, Srivastav, Lochner, Fahlgren, Nguyen, Fourrier, Burtenshaw, Larcher, Zhao, Zakka, Morlon, Raffel, von Werra, and Wolf]{allal2025smollm2}
L.~B. Allal, A.~Lozhkov, E.~Bakouch, G.~M. Blázquez, G.~Penedo, L.~Tunstall, A.~Marafioti, H.~Kydlícek, A.~P. Lajarín, V.~Srivastav, J.~Lochner, C.~Fahlgren, X.-S. Nguyen, C.~Fourrier, B.~Burtenshaw, H.~Larcher, H.~Zhao, C.~Zakka, M.~Morlon, C.~Raffel, L.~von Werra, and T.~Wolf.
\newblock {SmolLM2}: When smol goes big - {Data-Centric} training of a small language model.
\newblock \emph{arXiv}, 2025.
\newblock \doi{10.48550/arXiv.2502.02737}.
\newblock URL \url{https://arxiv.org/abs/2502.02737}.

\bibitem[Boratko et~al.(2020)Boratko, Li, O'Gorman, Das, Le, and McCallum]{boratko2020protoqa}
M.~Boratko, X.~Li, T.~O'Gorman, R.~Das, D.~Le, and A.~McCallum.
\newblock {ProtoQA}: A question answering dataset for prototypical {Common-Sense} reasoning.
\newblock In \emph{Proceedings of the 2020 Conference on Empirical Methods in Natural Language Processing, {EMNLP}}, pages 1122--1136, 2020.
\newblock \doi{10.18653/v1/2020.emnlp-main.85}.
\newblock URL \url{https://www.aclweb.org/anthology/2020.emnlp-main.85}.

\bibitem[Brown et~al.(2024)Brown, Juravsky, Ehrlich, Clark, Le, Ré, and Mirhoseini]{brown2024large}
B.~C.~A. Brown, J.~Juravsky, R.~Ehrlich, R.~Clark, Q.~V. Le, C.~Ré, and A.~Mirhoseini.
\newblock Large language monkeys: Scaling inference compute with repeated sampling.
\newblock \emph{arXiv}, 2024.
\newblock \doi{10.48550/arXiv.2407.21787}.
\newblock URL \url{https://arxiv.org/abs/2407.21787}.

\bibitem[Caccia et~al.(2020)Caccia, Caccia, Fedus, Larochelle, Pineau, and Charlin]{caccia2020language}
M.~Caccia, L.~Caccia, W.~Fedus, H.~Larochelle, J.~Pineau, and L.~Charlin.
\newblock Language {GANs} falling short.
\newblock In \emph{Proceedings of the 8th International Conference on Learning Representations, {ICLR}}, 2020.
\newblock \doi{10.48550/arxiv.1811.02549}.
\newblock URL \url{https://arxiv.org/abs/1811.02549}.

\bibitem[Chung et~al.(2025)Chung, Padmakumar, Roemmele, Sun, and Kreminski]{chung2025modifying}
J.~J.~Y. Chung, V.~Padmakumar, M.~Roemmele, Y.~Sun, and M.~Kreminski.
\newblock Modifying large language model {Post-Training} for diverse creative writing.
\newblock In \emph{Proceedings of the 2nd Conference on Language Modeling, {COLM}}, 2025.
\newblock \doi{10.48550/arXiv.2503.17126}.
\newblock URL \url{https://arxiv.org/abs/2503.17126}.

\bibitem[Freitag et~al.(2022)Freitag, Grangier, Tan, and Liang]{freitag2022high}
M.~Freitag, D.~Grangier, Q.~Tan, and B.~Liang.
\newblock High quality rather than high model probability: Minimum bayes risk decoding with neural metrics.
\newblock \emph{Transactions of the Association for Computational Linguistics}, 10:\penalty0 811--825, 2022.
\newblock \doi{10.1162/tacl_a_00491}.
\newblock URL \url{https://arxiv.org/abs/2111.09388}.

\bibitem[Helsgaun(2018)]{lkh_webpage}
K.~Helsgaun.
\newblock {LKH} ({Keld Helsgaun}), 2018.
\newblock URL \url{http://webhotel4.ruc.dk/~keld/research/LKH/}.

\bibitem[Ippolito et~al.(2019)Ippolito, Kriz, Sedoc, Kustikova, and Callison-Burch]{ippolito2019comparison}
D.~Ippolito, R.~Kriz, J.~Sedoc, M.~Kustikova, and C.~Callison-Burch.
\newblock Comparison of diverse decoding methods from conditional language models.
\newblock In \emph{Proceedings of the 57th Annual Meeting of the Association for Computational Linguistics, {ACL}}, pages 3752--3762, 2019.
\newblock \doi{10.18653/v1/p19-1365}.
\newblock URL \url{https://aclanthology.org/P19-1365}.

\bibitem[Jiang et~al.(2026)Jiang, Wang, Wang, Búš, Cheng, Wang, Liu, Li, Zeng, Liu, Liang, Xu, Hu, Zhang, and Dong]{jiang2026tapo}
M.~Jiang, Z.~Wang, Q.~Wang, P.~Búš, M.~Cheng, Y.~Wang, Q.~Liu, R.~Li, P.~Zeng, R.~Liu, A.~Liang, Y.~Xu, Y.~Hu, C.~Zhang, and Z.~Dong.
\newblock {TAPO}: Dynamic teacher and perturbed answer injection for policy optimization.
\newblock In \emph{Proceedings of the 40th AAAI Conference on Artificial Intelligence, {AAAI}}, pages 37462--37471, 2026.
\newblock \doi{10.1609/aaai.v40i44.41079}.
\newblock URL \url{https://doi.org/10.1609/aaai.v40i44.41079}.

\bibitem[Kool et~al.(2019)Kool, van Hoof, and Welling]{kool2019stochastic}
W.~Kool, H.~van Hoof, and M.~Welling.
\newblock Stochastic beams and where to find them: The {Gumbel-Top-k} trick for sampling sequences without replacement.
\newblock In \emph{Proceedings of the 36th International Conference on Machine Learning, {ICML}}, pages 3499--3508, 2019.
\newblock \doi{10.48550/arxiv.1903.06059}.
\newblock URL \url{https://arxiv.org/abs/1903.06059}.

\bibitem[Lamont et~al.(2025)Lamont, Walder, Dezfouli, Montague, and Norrish]{lamont20253d}
S.~Lamont, C.~Walder, A.~Dezfouli, P.~Montague, and M.~Norrish.
\newblock {3D-Prover}: Diversity driven theorem proving with determinantal point processes.
\newblock In \emph{Advances in Neural Information Processing Systems, {NeurIPS}}, pages 64741--64764, 2025.
\newblock \doi{10.52202/085713-2169}.
\newblock URL \url{https://arxiv.org/abs/2410.11133}.

\bibitem[Lamont et~al.(2026)Lamont, Walder, Montague, Dezfouli, and Norrish]{lamont2026free}
S.~Lamont, C.~Walder, P.~Montague, A.~Dezfouli, and M.~Norrish.
\newblock Free lunch for pass@k? low cost diverse sampling for diffusion language models.
\newblock \emph{arXiv}, 2026.
\newblock \doi{10.48550/arXiv.2603.04893}.
\newblock URL \url{https://arxiv.org/abs/2603.04893}.

\bibitem[Le et~al.(2026)Le, Jeon, Vu, Lai, and Yang]{le2026no}
T.-L.~V. Le, M.~Jeon, K.~Vu, V.~D. Lai, and E.~Yang.
\newblock No prompt left behind: Exploiting {Zero-Variance} prompts in {LLM} reinforcement learning via {Entropy-Guided} advantage shaping.
\newblock In \emph{Proceedings of the 14th International Conference on Learning Representations, {ICLR}}, 2026.
\newblock \doi{10.48550/arXiv.2509.21880}.
\newblock URL \url{https://arxiv.org/abs/2509.21880}.

\bibitem[Lee et~al.(2025)Lee, Chon, Jang, Lee, and Yu]{lee2025how}
S.~Lee, H.~Chon, J.~Jang, D.~Lee, and H.~Yu.
\newblock How diversely can language models solve problems? exploring the algorithmic diversity of {Model-Generated} code.
\newblock In \emph{Findings of the Association for Computational Linguistics: EMNLP 2025, {EMNLP}}, pages 152--167, 2025.
\newblock \doi{10.18653/v1/2025.findings-emnlp.10}.
\newblock URL \url{https://arxiv.org/abs/2503.00691}.

\bibitem[Li et~al.(2026{\natexlab{a}})Li, Zhang, Wang, Ma, Tang, Huang, and Duan]{li2026setpo}
C.~Li, Y.~Zhang, B.~Wang, G.~Ma, W.~Tang, H.~Huang, and N.~Duan.
\newblock {SetPO}: {Set-Level} policy optimization for {Diversity-Preserving} {LLM} reasoning.
\newblock In \emph{Proceedings of the 43rd International Conference on Machine Learning, {ICML}}, 2026{\natexlab{a}}.
\newblock \doi{10.48550/arXiv.2602.01062}.
\newblock URL \url{https://arxiv.org/abs/2602.01062}.

\bibitem[Li et~al.(2026{\natexlab{b}})Li, Zhan, Gandhi, Goodman, and Fox]{li2026quasimotto}
M.~Y. Li, A.~Zhan, K.~Gandhi, N.~D. Goodman, and E.~B. Fox.
\newblock {QuasiMoTTo}: {Quasi-Monte} carlo {Test-Time} scaling.
\newblock \emph{arXiv}, 2026{\natexlab{b}}.
\newblock \doi{10.48550/arXiv.2607.01179}.
\newblock URL \url{https://arxiv.org/abs/2607.01179}.

\bibitem[Li et~al.(2025)Li, Zhang, Yu, Saha, Khashabi, Weston, Lanchantin, and Wang]{li2025jointly}
T.~Li, Y.~Zhang, P.~Yu, S.~Saha, D.~Khashabi, J.~Weston, J.~Lanchantin, and T.~Wang.
\newblock Jointly reinforcing diversity and quality in language model generations.
\newblock \emph{arXiv}, 2025.
\newblock \doi{10.48550/arXiv.2509.02534}.
\newblock URL \url{https://arxiv.org/abs/2509.02534}.

\bibitem[Li et~al.(2026{\natexlab{c}})Li, Li, Fang, Ding, Li, Chen, Ma, Lyu, Li, Lin, Guo, Liu, and Chen]{li2026beyond}
X.~Li, Y.~Li, X.~Fang, S.~Ding, P.~Li, Y.~Chen, Y.~Ma, T.~Lyu, L.~Li, D.~Lin, Q.~Guo, Q.~Liu, and K.~Chen.
\newblock Beyond mode collapse: Distribution matching for diverse reasoning.
\newblock In \emph{Proceedings of the 43rd International Conference on Machine Learning, {ICML}}, 2026{\natexlab{c}}.
\newblock \doi{10.48550/arXiv.2605.19461}.
\newblock URL \url{https://arxiv.org/abs/2605.19461}.

\bibitem[Meister et~al.(2021)Meister, Forster, and Cotterell]{meister2021determinantal}
C.~Meister, M.~Forster, and R.~Cotterell.
\newblock Determinantal beam search.
\newblock In \emph{Proceedings of the 59th Annual Meeting of the Association for Computational Linguistics and the 11th International Joint Conference on Natural Language Processing (Volume 1: Long Papers), {ACL-IJCNLP}}, pages 6551--6562, 2021.
\newblock \doi{10.18653/v1/2021.acl-long.512}.
\newblock URL \url{https://aclanthology.org/2021.acl-long.512}.

\bibitem[Parashar et~al.(2025)Parashar, Singh, Amballa, Lai, and Rozonoyer]{parashar2025quasi}
A.~Parashar, A.~V. Singh, A.~Amballa, J.~Lai, and B.~Rozonoyer.
\newblock Quasi-random {Multi-Sample} inference for large language models.
\newblock In \emph{Proceedings of the ICLR 2025 Workshop on Frontiers in Probabilistic Inference: Learning Meets Sampling, {FIP}}, 2025.
\newblock URL \url{https://arxiv.org/abs/2411.06251}.

\bibitem[Qwen et~al.(2024)Qwen, :, An, Yang, Zhang, Hui, Zheng, Yu, Li, Liu, Huang, Wei, Lin, Yang, Tu, Zhang, Yang, Yang, Zhou, Lin, Dang, Lu, Bao, Yang, Yu, Li, Xue, Zhang, Zhu, Men, Lin, Li, Tang, Xia, Ren, Ren, Fan, Yang, Zhang, Wan, Liu, Cui, Zhang, and Qiu]{qwen2024qwen2}
Qwen, :, Y.~An, B.~Yang, B.~Zhang, B.~Hui, B.~Zheng, B.~Yu, C.~Li, D.~Liu, F.~Huang, H.~Wei, H.~Lin, J.~Yang, J.~Tu, J.~Zhang, J.~Yang, J.~Yang, J.~Zhou, J.~Lin, K.~Dang, K.~Lu, K.~Bao, K.~Yang, L.~Yu, M.~Li, M.~Xue, P.~Zhang, Q.~Zhu, R.~Men, R.~Lin, T.~Li, T.~Tang, T.~Xia, X.~Ren, X.~Ren, Y.~Fan, S.~Yang, Y.~Zhang, Y.~Wan, Y.~Liu, Z.~Cui, Z.~Zhang, and Z.~Qiu.
\newblock Qwen2.5 technical report.
\newblock \emph{arXiv}, 2024.
\newblock \doi{10.48550/arXiv.2412.15115}.
\newblock URL \url{https://arxiv.org/abs/2412.15115}.

\bibitem[Ruan et~al.(2025)Ruan, Li, Liu, Chen, Luo, Li, Liu, and Chen]{ruan2025g2}
Z.~Ruan, Y.~Li, Y.~Liu, Y.~Chen, W.~Luo, P.~Li, Y.~Liu, and G.~Chen.
\newblock G2: Guided generation for enhanced output diversity in {LLMs}.
\newblock In \emph{Proceedings of the 2025 Conference on Empirical Methods in Natural Language Processing, {EMNLP}}, pages 14116--14134, 2025.
\newblock \doi{10.18653/v1/2025.emnlp-main.713}.
\newblock URL \url{https://arxiv.org/abs/2511.00432}.

\bibitem[Sato(2022)]{sato2022word}
R.~Sato.
\newblock Word tour: One-dimensional word embeddings via the traveling salesman problem.
\newblock In \emph{Proceedings of the 2022 Conference of the North American Chapter of the Association for Computational Linguistics: Human Language Technologies, {NAACL-HLT}}, pages 2166--2172, 2022.
\newblock \doi{10.18653/v1/2022.naacl-main.157}.
\newblock URL \url{https://arxiv.org/abs/2205.01954}.

\bibitem[Shi et~al.(2020)Shi, Bieber, and Sutton]{shi2020incremental}
K.~Shi, D.~Bieber, and C.~Sutton.
\newblock Incremental sampling without replacement for sequence models.
\newblock In \emph{Proceedings of the 37th International Conference on Machine Learning, {ICML}}, pages 8785--8795, 2020.
\newblock \doi{10.48550/arxiv.2002.09067}.
\newblock URL \url{https://arxiv.org/abs/2002.09067}.

\bibitem[Shi et~al.(2025)Shi, Cui, Wu, Fang, Zhang, Li, Han, Zhu, Xu, and Zhou]{shi2025semantic}
W.~Shi, Y.~Cui, Y.~Wu, J.~Fang, S.~Zhang, M.~Li, S.~Han, J.~Zhu, J.~Xu, and X.~Zhou.
\newblock Semantic-guided diverse decoding for large language model.
\newblock In \emph{Advances in Neural Information Processing Systems, {NeurIPS}}, pages 148182--148229, 2025.
\newblock \doi{10.52202/085713-4953}.
\newblock URL \url{https://arxiv.org/abs/2506.23601}.

\bibitem[Troshin et~al.(2025)Troshin, Saparina, Fokkens, and Niculae]{troshin2025asking}
S.~Troshin, I.~Saparina, A.~Fokkens, and V.~Niculae.
\newblock Asking a language model for diverse responses.
\newblock In \emph{Proceedings of the 2nd Workshop on Uncertainty-Aware NLP, {UncertaiNLP}}, pages 66--72, 2025.
\newblock \doi{10.18653/v1/2025.uncertainlp-main.8}.
\newblock URL \url{https://arxiv.org/abs/2509.17570}.

\bibitem[Tuyls et~al.(2026)Tuyls, Foster, Krishnamurthy, and Ash]{tuyls2026representation}
J.~Tuyls, D.~J. Foster, A.~Krishnamurthy, and J.~T. Ash.
\newblock {Representation-Based} exploration for language models: From {Test-Time} to {Post-Training}.
\newblock In \emph{Proceedings of the 14th International Conference on Learning Representations, {ICLR}}, 2026.
\newblock \doi{10.48550/arXiv.2510.11686}.
\newblock URL \url{https://arxiv.org/abs/2510.11686}.

\bibitem[Vijayakumar et~al.(2018)Vijayakumar, Cogswell, Selvaraju, Sun, Lee, Crandall, and Batra]{vijayakumar2018diverse}
A.~K. Vijayakumar, M.~Cogswell, R.~R. Selvaraju, Q.~Sun, S.~Lee, D.~J. Crandall, and D.~Batra.
\newblock Diverse beam search for improved description of complex scenes.
\newblock In \emph{Proceedings of the 32nd AAAI Conference on Artificial Intelligence, {AAAI}}, pages 7371--7379, 2018.
\newblock \doi{10.1609/aaai.v32i1.12340}.
\newblock URL \url{https://doi.org/10.1609/aaai.v32i1.12340}.

\bibitem[Vilnis et~al.(2023)Vilnis, Zemlyanskiy, Murray, Passos, and Sanghai]{vilnis2023arithmetic}
L.~Vilnis, Y.~Zemlyanskiy, P.~Murray, A.~T. Passos, and S.~Sanghai.
\newblock Arithmetic sampling: Parallel diverse decoding for large language models.
\newblock In \emph{Proceedings of the 40th International Conference on Machine Learning, {ICML}}, pages 35120--35136, 2023.
\newblock \doi{10.48550/arXiv.2210.15458}.
\newblock URL \url{https://arxiv.org/abs/2210.15458}.

\bibitem[Wang et~al.(2025{\natexlab{a}})Wang, Pan, Linzen, and Black]{wang2025multilingual}
Q.~Wang, S.~Pan, T.~Linzen, and E.~Black.
\newblock Multilingual prompting for improving {LLM} generation diversity.
\newblock In \emph{Proceedings of the 2025 Conference on Empirical Methods in Natural Language Processing, {EMNLP}}, pages 6367--6389, 2025{\natexlab{a}}.
\newblock \doi{10.18653/v1/2025.emnlp-main.324}.
\newblock URL \url{https://arxiv.org/abs/2505.15229}.

\bibitem[Wang et~al.(2025{\natexlab{b}})Wang, Liu, Chen, Light, Liu, Chen, Zhang, and Cheng]{wang2025diversified}
T.~Wang, Z.~Liu, Y.~Chen, J.~Light, W.~Liu, H.~Chen, X.~Zhang, and W.~Cheng.
\newblock Diversified sampling improves scaling {LLM} inference.
\newblock \emph{arXiv}, 2025{\natexlab{b}}.
\newblock \doi{10.48550/arXiv.2502.11027}.
\newblock URL \url{https://arxiv.org/abs/2502.11027}.

\bibitem[Wang et~al.(2026{\natexlab{a}})Wang, Liu, Chen, Light, Liu, Chen, Zhang, and Cheng]{wang2026effectsampling}
T.~Wang, Z.~Liu, Y.~Chen, J.~Light, W.~Liu, H.~Chen, X.~Zhang, and W.~Cheng.
\newblock On the effect of sampling diversity in scaling llm inference.
\newblock In \emph{Proceedings of the 42nd Conference on Uncertainty in Artificial Intelligence, {UAI}}, 2026{\natexlab{a}}.
\newblock URL \url{https://arxiv.org/abs/2502.11027}.

\bibitem[Wang et~al.(2026{\natexlab{b}})Wang, Yang, Huang, Chen, May, and Lee]{wang2026optimizing}
Y.~Wang, C.~Yang, T.~Huang, M.~Chen, J.~May, and M.~Lee.
\newblock Optimizing diversity and quality through {Base-Aligned} model collaboration.
\newblock In \emph{Proceedings of the 43rd International Conference on Machine Learning, {ICML}}, 2026{\natexlab{b}}.
\newblock \doi{10.48550/arXiv.2511.05650}.
\newblock URL \url{https://arxiv.org/abs/2511.05650}.

\bibitem[Wong et~al.(2025)Wong, Orlovskiy, Shypula, Luo, Seshia, and Gonzalez]{wong2025simplestrat}
J.~Wong, Y.~Orlovskiy, A.~Shypula, M.~Luo, S.~A. Seshia, and J.~E. Gonzalez.
\newblock {SimpleStrat}: Diversifying language model generation with stratification.
\newblock In \emph{Advances in Neural Information Processing Systems, {NeurIPS}}, pages 116906--116935, 2025.
\newblock \doi{10.52202/085713-3897}.
\newblock URL \url{https://arxiv.org/abs/2410.09038}.

\bibitem[Wu et~al.(2026)Wu, Wan, Yu, Yang, Huang, Tsang, and You]{wu2026time}
J.~Wu, Z.~Wan, X.~Yu, Y.~Yang, Y.~Huang, I.~W. Tsang, and Y.~You.
\newblock {Time-Annealed} perturbation sampling: Diverse generation for diffusion language models.
\newblock \emph{arXiv}, 2026.
\newblock \doi{10.48550/arXiv.2601.22629}.
\newblock URL \url{https://arxiv.org/abs/2601.22629}.

\bibitem[Xu et~al.(2026)Xu, Chen, Qiu, Yan, Luo, Cheng, He, and Tong]{xu2026prune}
H.~Xu, S.~Chen, R.~Qiu, Y.~Yan, C.~Luo, M.~X. Cheng, J.~He, and H.~Tong.
\newblock Prune as you generate: Online rollout pruning for faster and better {RLVR}.
\newblock In \emph{Proceedings of the 64th Annual Meeting of the Association for Computational Linguistics (Volume 1: Long Papers), {ACL}}, pages 13876--13893, 2026.
\newblock \doi{10.18653/v1/2026.acl-long.632}.
\newblock URL \url{https://arxiv.org/abs/2603.24840}.

\bibitem[Yao et~al.(2025)Yao, Cheng, Wu, Wu, and Tan]{yao2025diversity}
J.~Yao, R.~Cheng, X.~Wu, J.~Wu, and K.~Tan.
\newblock {Diversity-Aware} policy optimization for large language model reasoning.
\newblock In \emph{Advances in Neural Information Processing Systems, {NeurIPS}}, pages 94801--94826, 2025.
\newblock \doi{10.52202/085713-3169}.
\newblock URL \url{https://arxiv.org/abs/2505.23433}.

\bibitem[Yu et~al.(2025)Yu, Zhang, Zhu, Yuan, Zuo, Yue, Dai, Fan, Liu, Liu, Liu, Liu, Lin, Lin, Ma, Sheng, Tong, Zhang, Zhang, Zhang, Zhang, Zhu, Zhu, Chen, Chen, Wang, Yu, Song, Wei, Zhou, Liu, Ma, Zhang, Yan, Wu, and Wang]{yu2025dapo}
Q.~Yu, Z.~Zhang, R.~Zhu, Y.~Yuan, X.~Zuo, Y.~Yue, W.~Dai, T.~Fan, G.~Liu, J.~Liu, L.~Liu, X.~Liu, H.~Lin, Z.~Lin, B.~Ma, G.~Sheng, Y.~Tong, C.~Zhang, M.~Zhang, R.~Zhang, W.~Zhang, H.~Zhu, J.~Zhu, J.~Chen, J.~Chen, C.~Wang, H.~Yu, Y.~Song, X.~Wei, H.~Zhou, J.~Liu, W.-Y. Ma, Y.-Q. Zhang, L.~Yan, Y.~Wu, and M.~Wang.
\newblock {DAPO}: An {Open-Source} {LLM} reinforcement learning system at scale.
\newblock In \emph{Advances in Neural Information Processing Systems, {NeurIPS}}, pages 113222--113244, 2025.
\newblock \doi{10.52202/085713-3775}.
\newblock URL \url{https://arxiv.org/abs/2503.14476}.

\bibitem[Zhu et~al.(2026)Zhu, Khanh, Cheung, Yue, and Nguyen]{zhu2026exploring}
D.~Zhu, L.~T.~H. Khanh, A.~Y.-M. Cheung, M.-C. Yue, and V.~A. Nguyen.
\newblock Exploring diverse generation paths via inference-time stiefel activation steering.
\newblock In \emph{Proceedings of the 14th International Conference on Learning Representations, {ICLR}}, 2026.
\newblock \doi{10.48550/arXiv.2601.22010}.
\newblock URL \url{https://arxiv.org/abs/2601.22010}.

\end{thebibliography}
